\documentclass[submission,copyright,creativecommons]{eptcs}

\providecommand{\event}{Post Proceedings of ThEdu'21}

\usepackage{breakurl}
\usepackage[english]{babel}
\usepackage{amsthm} 
\usepackage[portraitbreqn]{gclc_proof}
\usepackage{slashed}
\usepackage{graphicx}
\usepackage{color}
\usepackage{hyperref}
\usepackage{listings}

\newcounter{conjecturas}

\def\seccaonovaconjectura{\protect%
  \refstepcounter{conjecturas}%
  {Problem \theconjecturas}%
  }

\def\novaconjectura{\protect%
  {Problem \theconjecturas}%
  }

\newcommand{\JGEX}{\textit{JGEX\/}}
\newcommand{\GeoGebra}{\textit{GeoGebra\/}}
\newcommand{\GCLC}{\textit{GCLC\/}}
\newcommand{\ProverNine}{\textit{Prover9\/}}

\newtheorem*{theorem*}{Theorem}

\title{Four Geometry Problems to Introduce Automated Deduction in
  Secondary Schools}

\author{Pedro Quaresma\thanks{This work is funded
    by national funds through the FCT - Foundation for Science and
    Technology, I.P., within the scope of the project CISUC -
    UID/CEC/00326/2020 and by European Social Fund, through the
    Regional Operational Program Centro 2020.}
  \institute{
    CISUC / Department of Mathematics, \\
    University of Coimbra, Portugal}
  \email{pedro@mat.uc.pt}
  \and
  Vanda Santos
  \thanks{This work is funded
    by national funds through the FCT - Foundation for Science and
    Technology, I.P., within the scope of the project UIDB/00194/2020 and in the scope of the framework contract foreseen in the numbers 4, 5 and 6 of the article 23, of the Decree-Law 57/2016, of August 29, changed by Law 57/2017, of July 19. }
  \institute{
    CIDTFF / University of Aveiro \\
    and CISUC, Portugal}
  \email{vandasantos@ua.pt}
}

\def\authorrunning{P. Quaresma \& V. Santos}

\def\titlerunning{Four Geometry Problems to Introduce Automated Deduction}

\begin{document}

\lstdefinelanguage{GDDM}{
  extendedchars=true,
  inputencoding=latin1,
  morekeywords={fof,axiom},
  basicstyle=\footnotesize
}

\lstdefinelanguage{Algoritmo}{
  extendedchars=true,
  inputencoding=latin1,
  morekeywords={WHILE,DO,ENDWHILE,IF,THEN,ELSE,ENDIF},
  basicstyle=\footnotesize
}

\lstdefinelanguage{FOF}{
  extendedchars=true,
  inputencoding=latin1,
  morekeywords={include,fof,midp,para,conjecture},
  basicstyle=\footnotesize
}

\lstdefinelanguage{ProverNine}{
  extendedchars=true,
  inputencoding=latin1,
  morekeywords={all,midp,para},
  basicstyle=\footnotesize
}

\lstdefinelanguage{maude} {
  basicstyle=\scriptsize,
  upquote=true,
  columns=fullflexible,
  showstringspaces=false,
  extendedchars=true,
  breaklines=true,
  showtabs=false,
  showspaces=false,
  showstringspaces=false,
  identifierstyle=\ttfamily,
  keywordstyle=\color[rgb]{0,0,1},
  ndkeywordstyle=\color[RGB]{165, 42, 42},
  commentstyle=\color[rgb]{1,0,0},
  stringstyle=\color[RGB]{77,144,45},
  frame = single,
  morecomment=[l]{***},
  ndkeywords = {Int, Bool, String},
  inputencoding=latin1,
  keywords={in, load, pr, protecting, sort, sorts, op, ops, var, vars, eq, cq, ceq, endfm, fmod, is, mod, endm, load, =, ==, =/= } 
}

\maketitle

\begin{abstract}
  The introduction of automated deduction systems in secondary schools
  face several bottlenecks, the absence of the subject of rigorous
  mathematical demonstrations in the curricula, the lack of knowledge
  by the teachers about the subject and the difficulty of tackling the
  task by automatic means.

  Despite those difficulties we claim that the subject of automated
  deduction in geometry can be introduced, by addressing it in
  particular cases: simple to manipulate by students and teachers and
  reasonably easy to be dealt by automatic deduction tools.

  The subject is discussed by addressing four secondary
  schools geometry problems: their rigorous proofs, visual proofs,
  numeric proofs, algebraic formal proofs, synthetic formal proofs, or
  the lack of them. For these problems we discuss a lesson plan to
  address them with the help of Information and Communications
  Technology, more specifically, automated deduction tools.
 \end{abstract}

\section{Introduction}
\label{sec:introduction}

The introduction of automated deduction systems in secondary schools
face several bottlenecks, the absence of the subject, rigorous
mathematical demonstrations, not to mention formal proofs, in many of
the national curricula, the lack of knowledge (and/or training) by the
teachers about the subject~\cite{Santos2021} and the difficulty of
tackling the task by automatic means~\cite{Chou2001}, are the most
important in our opinion.
  
In the area of geometry there are now a large number of computational
tools that can be used to perform many different tasks, dynamic
geometry systems (DGS), computer algebra systems (CAS), geometry
automatic theorem provers (GATP) among
others~\cite{Quaresma2017}. These tools can be useful to address the
subject of proofs in secondary schools.

The subject of automated deduction has been progressing for many years
now and has already attained some remarkable progresses,\footnote{See
  the series of proceedings of the CADE conferences,
  \url{http://www.cadeinc.org/}} but there are many obstacles still to
be addressed when we speak about the use of geometry automated theorem
provers (GATP) in secondary school classes: many of the GATP were
specified and implemented for research purposes and their use and
outputs are designed to be understood by experts, not for a secondary
teacher/student; some of the methods (e.g. the algebraic methods) do
not produce a geometric proofs; many (if not all) of the synthetic
methods produce a proof, but in an axiomatic system different from the
usually used by the secondary teachers/students; many of the proposed
problems, e.g. those involving inequalities,\footnote{In another
  contribution to this volume, \emph{Symbolic comparison of geometric
    quantities in GeoGebra}, Zolt{\'a}n Kov{\'a}cs and R{\'o}bert
  Vajda, address this issue.} are still not addressed by the GATP; in
many cases the time needed for a GATP to produce an answer is to long
to be useful in a classroom scenario; last (but not the least), are
formal proofs what a secondary teacher/student need?

An answer to this last question is being given by interactive,
tutorial, deduction systems. From the secondary schools point of view, a generic fully automated theorem prover may not be the most
  appropriated choice, a tutorial system, with a minimum set of rules
  may be a better choice. What is needed is a system that allows the
  students to make conjectures and proofs, in natural language and with a
  set of rules, the closest possible with their usual practice. The tutorial systems are much close to provide such approach then the generic automated theorem provers.

\textit{Overview of the paper.} The paper is organised as follows:
first, in Section~\ref{sec:gatpsSecondaySchools},
~four secondary school geometry problems are introduced and the
development of proofs with the help of Information and Communications
  Technology (ICT) tools, is discussed. In
Section~\ref{sec:newApproaches} the case for a new approach, a
tailored, to the needs of secondary schools, axiom system approach is
discussed. In Section~\ref{sec:conclusions} final conclusions are
drawn. In appendix~\ref{sec:classPlan} a lesson plan for a classroom,
where automated deduction tools could be introduced, is given and
finally in appendix~\ref{sec:ICTproblem1} the details of the use of
the ICT tools for the problem 1 are shown.

\section{Use of GATP in Secondary Schools}
\label{sec:gatpsSecondaySchools}

In the recent book edited by Gila Hanna, David Reid, and Michael de
Villiers, \emph{Proof Technology in Mathe\-matics Research and
  Teaching}~\cite{Hanna2019} many issues related to the use of
automated deduction me\-thods and tools are discussed. In its four
parts: \emph{Strengths and limitations of automatic theorem provers};
\emph{Theoretical perspectives on computer-assisted proving};
\emph{Suggestions for the use of proof software in the classroom};
\emph{Classroom experience with proof software}, we have a
comprehensive coverage of the area. The present study it is intended
as a new contribution to the area, in a more hands-on approach.

In appendix~\ref{sec:classPlan} a lesson plan for the study of the
proofs in geometry is presented and in the following sections, four
problems, that could be the subject of the lesson plan, are presented
and the use of ICT tools for the development of their formal proofs is
discussed. For each of these problems its adequacy for the intended
goal is discussed.

\subsection{Dynamic Geometry and Automated Deduction Tools}
\label{sec:DGSandATP}

Dynamic geometry systems and geometry automated theorem provers are
tools that enable teachers and students to explore existing knowledge,
to create new constructions and conjecture new properties.  Dynamic
geometry systems allow building geometric constructions from free
objects and elementary constructions, it is possible to manipulate the
free objects (objects universally quantified), preserving the
geometric properties of the constructions. Moving the free points, on the one hand, we can conjecture that a given property is true.  Although those
manipulations are not formal proofs because only a finite set of
positions are being considered and because visualisation can be
misleading, they provide a first clue to the truthfulness of a given
geometric conjecture. On the other hand, geometry automated theorem
provers allow the development of formal proofs. Based on different
approaches (e.g. algebraic, synthetic) they allow its users to check
the soundness of a construction and also, in some cases, to create
formal proofs for a given geometric conjectures.

In the following we focus on some ICT tools that can be used to help
students and teachers in the exploration of geometric
proofs. \emph{GeoGebra}\footnote{\url{https://www.geogebra.org}} is a
very well-known DGS, but it also contains several GATP that can be
used under a portfolio
strategy~\cite{Botana2015a,Kovacs2015,Kovacs2020,Nikolic2018}. \emph{Java
  Geometry Expert}
(JGEX),\footnote{\url{https://github.com/yezheng1981/Java-Geometry-Expert}}
a GATP and DGS combo with a strong focus on the formal proofs engines:
the Wu, Gr{\"o}bner bases, area, full-angle and deductive database
methods. It is possible to add a conjecture to a given geometric
construction and ask for its proof with natural and visual language
renderings~\cite{Ye2011}. Geometry Constructions $\rightarrow$ LaTeX
Converter
(GCLC),\footnote{\url{http://poincare.matf.bg.ac.rs/~janicic/gclc/}}
a GATP, with a graphics engine, for the Wu, Gr{\"o}bner bases and area
methods. It is possible to add a conjecture to a given geometric
construction (with a graphical rendering) and ask for its proof with
natural language
rendering~\cite{Janicic2010,Janicic2006c}. \emph{Prover
  9}\footnote{\url{https://www.cs.unm.edu/~mccune/prover9/}} is a
generic automated theorem prover for first-order and equational logic.
It is not specific to geometry, but can be used in geometry, provided
a specific set of axioms (e.g. Tarski,
Full-angle~\cite{Chou2000,Quaife1989}), is given.

\subsection{The Four Geometric Problems}
\label{sec:TheFourProblems}

Using a DGS (\GeoGebra, \JGEX) it is always possible to dynamically
manipulate the geometric construction, getting a visual check for the
conjecture at hand. Both programs also allow to perform a numerical
check. In either cases we have checks, not proofs. For the next step
we need a GATP (incorporated in the DGS or not) or a generic ATP
(e.g. \ProverNine).

The GATP came in different forms, algebraic (\GeoGebra, \GCLC, \JGEX),
synthetic: by a geometric proof method (\GCLC, \JGEX); or by a logical
proof method (\ProverNine).

The algebraic provers (Wu's method, Gr{\"o}bner Basis method, etc.) are
not that useful in the secondary classroom. Most of the times the
output of a formal proof is only a yes/no answer, even when a proof
script is produced, it is too complex and not related with geometric
reasoning, to be used in a secondary classroom.

The geometric provers (area method, full-angle method, database
method, etc.) provide a geo\-me\-tric proof script, readable, but based in
axiomatic systems that are not the usually used in secondary
classrooms (see Appendix~\ref{sec:ICTproblem1}), so its usefulness
is diminished by that fact.

{\JGEX} has a very interesting feature, it links the formal proof with
the geometric construction with visual elements (colours, blinking
effects, etc.). It is not being developed at this moment, but given
that the code is now open source, available at
\emph{GitHub},\footnotemark[3] maybe its development can be resumed in the 
future.

Generic ATP, like \ProverNine, using an axiomatic system for
geometry,\footnote{The geometric deductive database method was used in
  the proof attempts in this paper} can also be used, the proof is
done by refutation using the resolution method, so it can be used by
teachers (maybe with the help of an expert) as a guide to construct 
a formal proof, to be rendered during a class, to the students.

\paragraph{\seccaonovaconjectura}
\label{sec:problem1}

\begin{theorem*} 
  Show that for any given convex quadrilateral, $[ABCD]$, that
  $[EFGH]$, where each of the points is the mid point of a segment in
  $[ABCD]$, is a parallelogram.

  \begin{figure}[hbtp]
    \centering
    \includegraphics[width=0.55\textwidth]{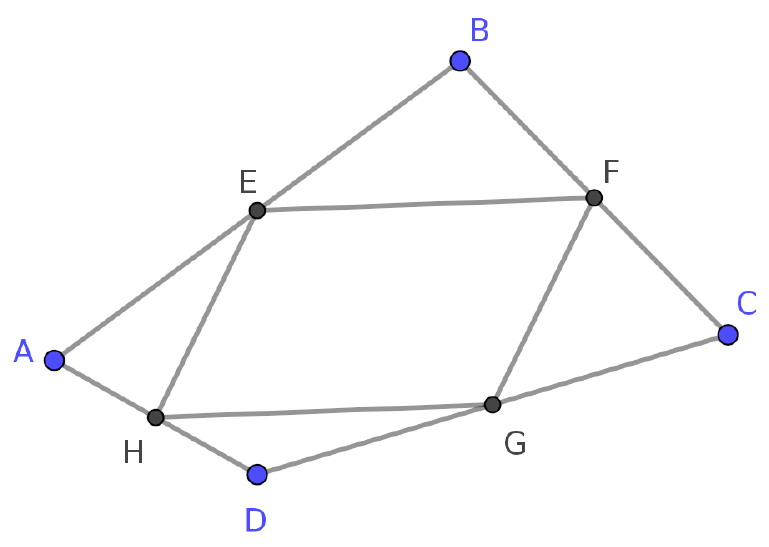}      
    \caption{\novaconjectura}
    \label{fig:convexquadrilateral}
  \end{figure}
   
\end{theorem*}

\begin{proof}
  Consider any convex quadrilateral $[ABCD]$, as shown in
  figure~\ref{fig:convexquadrilateral} where $E$, $F$, $G$ and $H$ are
  the midpoints of $[AB]$, $[BC]$, $[CD]$ and $[DA]$ respectively.
  Draw the diagonal $[AC]$ (see Figure~\ref{fig:parallegoram1}), so
  the triangles $\Delta ABC$ and $\Delta EBF$ are similar since they
  have a common angle, $\alpha = \angle ABC$, and from one to the
  other the two pairs of sides $[AB]$, $[EB]$ and $[BC]$, $[FB]$
  (which are contained on the sides of the angle) are directly 
  proportional, $\frac{|AB|}{|EB|} = \frac{|CB|}{|FB|}$
  (Side-Angle-Side case). 

  \begin{figure}[hbtp!]
    \centering
    \includegraphics[width=0.55\textwidth]{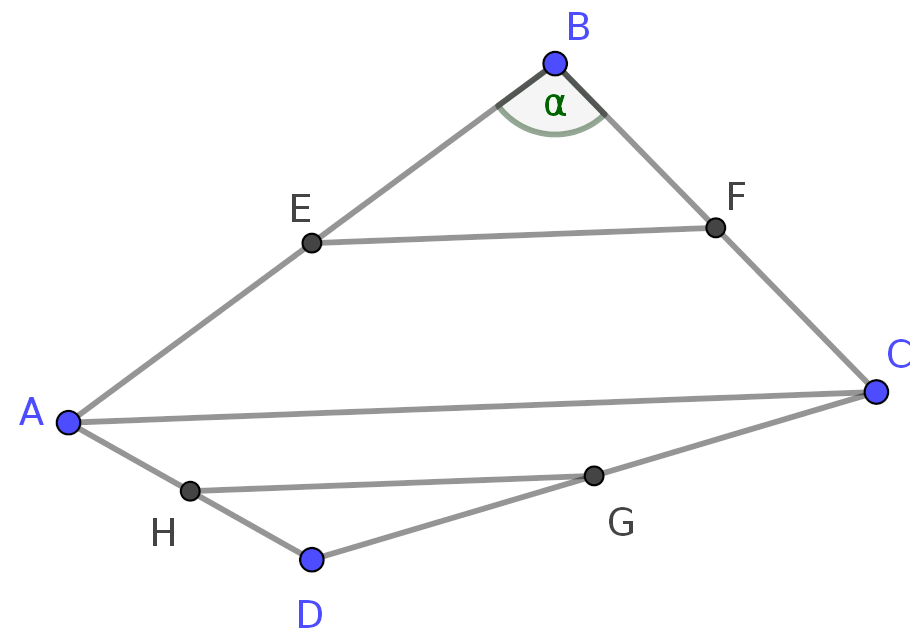}      
    \caption{Sides $[EF]$ and $[HG]$ are parallel}
    \label{fig:parallegoram1}
  \end{figure}

  Given that, the angles $\angle BAC$ and $\angle BEF$ are congruent and as these angles have a common side, then $[EF] \| [AC]$.
  
  Following an analogous reasoning, we have that $[AC] \| [GH]$, since
  the triangles $\Delta DHG$ and $\Delta DAC$ are similar.

  We have that $[EF] \| [AC]$ and $[AC] \| [GH]$ then by the
  transitive property of the parallelism relationship it can be
  concluded that $[EF] \| [GH]$.

  \begin{figure}[hbtp!]
    \centering
    \includegraphics[width=0.55\textwidth]{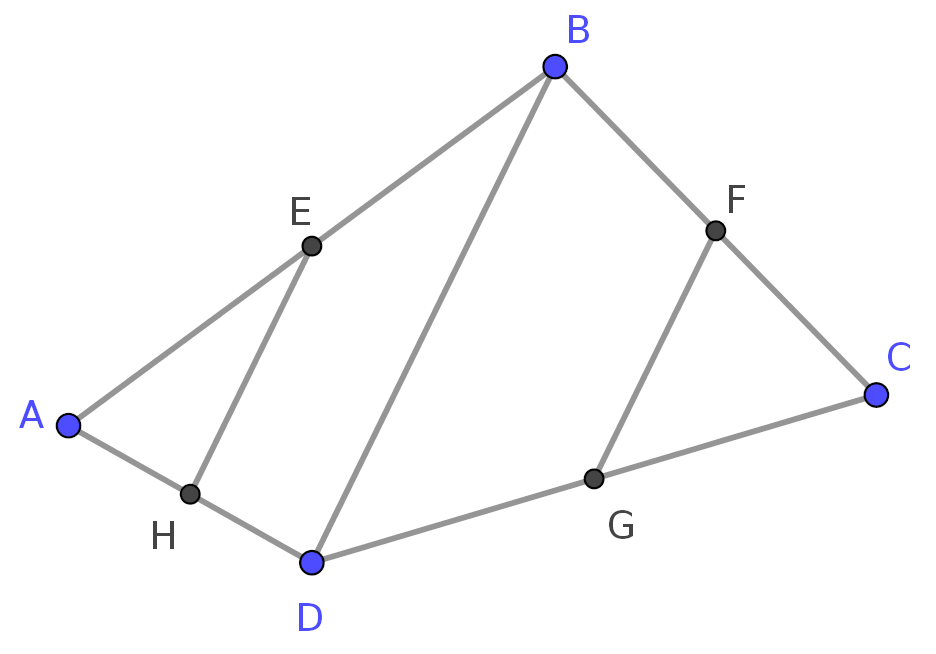}      
    \caption{Sides $[FG]$ and $[EH]$ are parallel}
    \label{fig:parallegoram2}
  \end{figure}
  
  Draw the diagonal $[BD]$ (see
  Figure~\ref{fig:parallegoram2}). Following a reasoning similar to
  the previous step it can be proved that $[FG] \| [EH]$.
  
  Since $[EF] \| [GH]$ and $[EH] \| [FG]$ it can be concluded that, the vertices of the midpoints of any quadrilateral convex are vertices of a parallelogram. 
\end{proof}

Formal proofs, done with the help of ICT tools:

\begin{description}
\item[GeoGebra] visual dynamic and numerical checks. A formal
  (algebraic) proof is possible, via the \textit{Prove} or
  \textit{ProveDetails} commands, with a yes/no answer.
\item[JGEX] visual dynamic and numerical checks. Geometry deductive
  database method; full-angle method; Gr{\"o}bner bases method; Wu's
  method (less than 1 second) --- formal proofs with visual helps.
\item[GCLC] area method, 0.001 seconds; Wu's method, 0.051 seconds;
  Gr{\"o}bner bases method 0.132 seconds --- with formal proof scripts.
\item[Prove9] geometry deductive database axioms, resolution method
  0.02 seconds --- formal proof script.
\end{description}

For this problem we have an array of choices, the {\GeoGebra} and the
{\JGEX} tools are the most complete: DGS; checks; formal proofs, in
one single package. {\JGEX} provides proof scripts and visual
aids. {\GeoGebra} is the best known by teachers and students. {\GCLC}
can also be considered, although it lacks the dynamic aspect and the
area method is based on an axiom system different from the usual used
in secondary schools.  {\ProverNine}, with a proper axiom system for
geometry, is capable of develop formal proofs, producing a readable
proof script. But it is a technical tool, only usable by experts. In
appendix~\ref{sec:ICTproblem1} the proofs done by the different ICT
tools are given in more detail.

\paragraph{\seccaonovaconjectura}
\label{sec:problem2}

\begin{theorem*} 
  Consider the convex quadrilateral $[ABCD]$, like the one represented
  in the figure~\ref{fig:quadrilatero2}, Assuming that $|BD|> |BC|$ and
  that $\alpha = \angle CAB > \angle ABC = \beta$, show that $|BD| > |AC|$. 
  
  \begin{figure}[hbtp]
    \centering
    \includegraphics[width=0.6\textwidth]{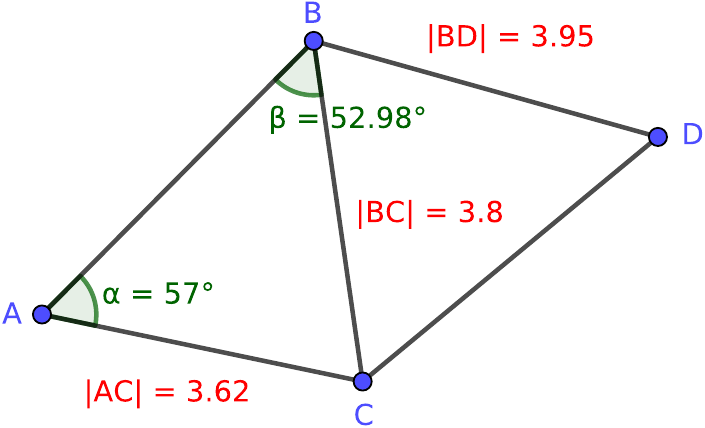}      
    \caption{\novaconjectura}
    \label{fig:quadrilatero2}
  \end{figure}
\end{theorem*}

\begin{proof}
  Given that in any triangle, the angle with the greatest amplitude is
  opposed to the side with the longest length, we have that in the
  $\Delta ABC$, $|BC| > |AC|$.

  Given that $|BD| > |BC|$, by hypothesis and, $|BC| > |AC|$, just
  proved above, due to the transitive property of the '$>$' order
  relation of we have $|BD| > |AC|$.
\end{proof}

Proofs with the help of ICT tools:

\begin{description}
\item[GeoGebra] visual dynamic and numerical checks. Using {\GeoGebra}'s
  new tool, \textit{Discover}, it is possible a formal proof, but without
  a proof script.
\item[JGEX] visual dynamic and numerical checks.
\end{description}

The different formal proving approaches are not possible  (apart the
new algebraic approach built-in in GeoGebra) given that
the current axiomatic sys\-tems/meth\-ods do not deal with inequalities.

\paragraph{\seccaonovaconjectura}
\label{sec:problem3}

\begin{theorem*}
  Knowing that $BC \| FG$, $|BC| = |CD|$ and $BD \|
  EF$ (see Fig.~\ref{fig:paralelas}), show that $G$ is equidistant
  from $E$ and $F$.

  \begin{figure}[hbtp]
    \centering
    \includegraphics[width=0.85\textwidth]{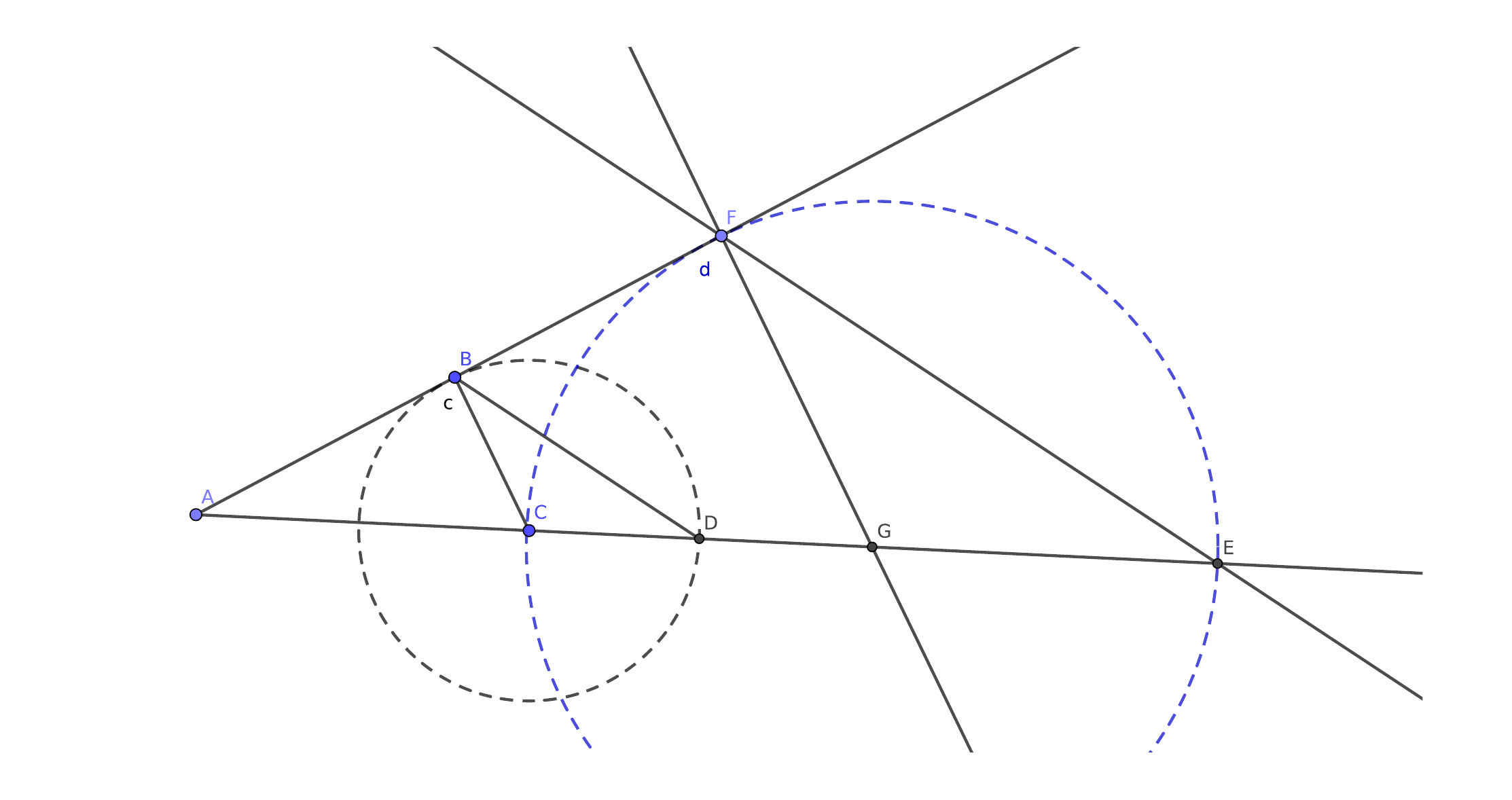}      
    \caption{\novaconjectura}
    \label{fig:paralelas}
  \end{figure}
\end{theorem*}

\begin{proof}
  $\angle BDC \cong \angle FEG$, since they are corresponding angles
  of two parallel lines ($BD \| EF)$ crossed by another line ($AE$).

  Similarly $\angle FGE \cong \angle BCD$, since they are corresponding angles
  of two parallel lines ($BC \| FG)$ crossed by another line ($AE$).

  Given that, $\Delta [BCD]$ e $\Delta [FGE]$ are similar (case
  angle-angle). In this way, we have:
  \begin{equation}
    \label{eq:similartriangles}
    \frac{|FG|}{|BC|} = \frac{|EG|}{|CD|}
  \end{equation}
  Since $C$ is equidistant from $B$ and $D$ we have that $|BC|=|CD|$,
  from this and equation~\ref{eq:similartriangles}, it can be
  concluded that $|FG|=|EG|$, so $G$ is equidistant from $E$ and $F$.
\end{proof}

Proofs with the help of ICT tools:

\begin{description}
\item[GeoGebra] visual dynamic and numerical checks. No formal proof
  given  that the \texttt{Prove} command do not have the equal length
  conjecture. 
\item[JGEX] visual dynamic and numerical checks. Geometry deductive
  database method, full-angle method; Wu's method; Gr{\"o}bner basis
  method (less than 1 second) --- formal proofs with visual helps.
\item[GCLC] area method, not proved (time limit); Wu's method, 0.005;
  Gr{\"o}bner basis method, 0.005.
\item[Prove9]  geometry deductive database axioms, resolution method
  0.04 seconds --- formal proof script.
\end{description}

Again, an array of choices, the {\JGEX} tool is the most complete:
DGS; checks; formal proofs, in one single package. {\JGEX} provides
proof scripts and visual aids.

\paragraph{\seccaonovaconjectura}
\label{sec:problem4}

\begin{theorem*}
  Show that the sum of the amplitudes of the internal angles of a
  triangle is $180^{\mathrm{o}}$.

  \begin{figure}[hbtp!]
    \centering
    \includegraphics[width=0.6\textwidth]{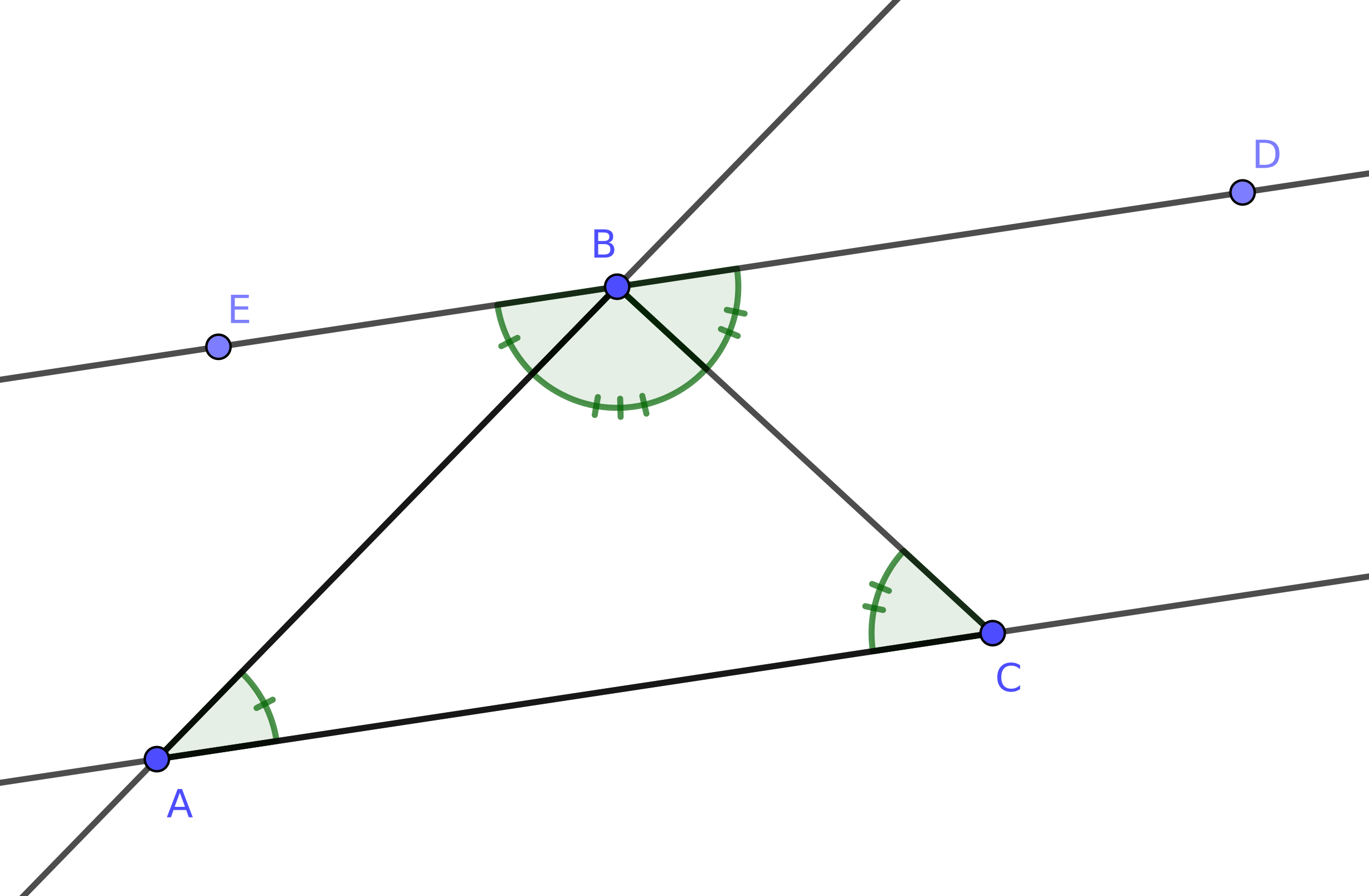}      
    \caption{\novaconjectura}
    \label{fig:problema5}
  \end{figure}
\end{theorem*}

\begin{proof}
  A line is drawn parallel to $AC$ and containing the vertex $B$. Let
  $D$ and $E$ be two points in that line. $\angle BAC \cong \angle EBA$
  given that they are alternate angles. Similarly
  $\angle BCA \cong \angle CBD$. The sum of the amplitudes of the
  $\angle CBD$, $\angle CBA$ and $\angle ABE$ angles is $180^{\mathrm{o}}$.

  Thus, the sum of the interior angles of a triangle is $180^{\mathrm{o}}$.
\end{proof}

Proofs with the help of ICT tools:

\begin{description}
\item[GeoGebra] visual dynamic and numerical checks. 
\item[JGEX] visual dynamic and numerical checks. 
\end{description}

The different formal proving approaches are not possible given that
the current axiomatic systems and/or methods do not deal with the sum of
angles.

\section{A Tailored Axiom System Approach}
\label{sec:newApproaches}

It seems that, more than a generic (research) tool that can prove
everything (or try to), a rule based approach will be more
appropriate.  Rule based approaches explore the possibility of
building a sound, not necessarily complete, axiom system. The idea is
to have a minimal set of axioms, lemmas, and rules of inference that
can characterize a given sub-area of geometry.  This approach has
already been tried in tutorial systems like the
\emph{QED-Tutrix}~\cite{Font2018} and
\emph{iGeoTutor}~\cite{Wang2015}. For the goal of introducing
automated deduction in secondary schools, it seems the appropriated
approach to pursue.

\subsection{Deductive Databases in Geometry}
\label{sec:gddm}

A deductive database is a database in which new facts may be derived
from facts that were explicitly introduced~\cite{Gallaire1984}. In a
deductive database the elementary facts (axioms and lemmas) are facts
in the database and each n-ary predicate is associated with a
n-dimensional relation. A data-based search strategy is usually used,
that is, from the old facts, new facts are deduced. 
 
When considering the geometry deductive database method, the initial
construction is used to set the initial new-fact-list, and all the
axioms and lemmas are converted in tables and relations. The proof
will cycle, deriving new facts from the initial list, until no new
fact is derived~\cite{Chou2000}.

Apart the rules of full-angle method ~\cite{Chou1996d,Chou2000} (see
Table~\ref{tbl:gddm}, for an excerpt) it is possible to add high-level
lemmas to have formal proofs that are more close to the practice in
secondary schools.

\begin{table}[htbp!]
\begin{lstlisting}[language=GDDM,basicstyle=\footnotesize,frame=single]
fof(ruleD1,axiom,( ! [A,B,C] : (coll(A,B,C) => coll(A,C,B)))).
fof(ruleD2,axiom,( ! [A,B,C] : (coll(A,B,C) => coll(B,A,C)))).
fof(ruleD3,axiom,( ! [A,B,C,D] : 
                     ((A!=B & coll(A,B,C) & coll(A,B,D)) => coll(C,D,A)))).
fof(ruleD4,axiom,( ! [A,B,C,D] : (para(A,B,C,D) => para(A,B,D,C)))).
fof(ruleD5,axiom,( ! [A,B,C,D] : (para(A,B,C,D) => para(C,D,A,B)))).
(...)
\end{lstlisting}
  \caption{Full-angle Rules}
  \label{tbl:gddm}
\end{table}

\subsection{Maude Equational (and Rewriting) Logic Programming}
\label{sec:maude}

The \emph{Maude system} is an implementation of rewriting logic. Using
\emph{Maude} it is possible to implement a given axiomatic for
geometry, e.g. Tarski's axiom system~\cite{Quaife1989} (see
Table~\ref{tbl:maude}, for an excerpt). Like in the deductive
databases approach it is possible to add high-level lemmas.

\begin{table}[htbp!]
\begin{lstlisting}[language = Maude,frame=single,basicstyle=\footnotesize]
*** System Tarski over G3cp
fmod FORMULA is
pr QID .  *** Maude Qualified Identifiers.
sorts Prop Formula Point Segments. *** Atomic propositions, Formulas, Points, Segments
subsort Qid < Prop < Formula .
subsort Qid < Point .
*** Tarski geometry primitive relations
op p : -> Point [ctor] .
op pl : -> Point [ctor] .
op pll : -> Point [ctor] .
op _*_ : Point Point -> Segment [ctor comm] .
op betweenness : Point Point Point -> Prop [ctor]  .
op equidistance : Segment Segment -> Prop [ctor comm] .
op innerPasch : Point Point Point Point Point -> Point .
endfm
mod Tarski is
*** Tarski Geometry (Art Quaife (1989), JAR 5, 97--118.
*** A7 Inner Pasch
rl [ip1] : C , betweenness(U,V,W), betweenness(Y,X,W) |-- betweenness(V,innerPasch(U,V,W,X,Y),Y) , Cl => proved .
rl [ip2] : C , betweenness(U,V,W), betweenness(Y,X,W) |-- betweenness(X,innerPasch(U,V,W,X,Y),U) , Cl => proved .
\end{lstlisting}
  \caption{Tarski's Axioms in Maude}
  \label{tbl:maude}
\end{table}

\subsection{Tutorial Systems}
\label{sec:tutorial}

A tutor system consists of an artificial tutor that accompanies the
student in solving problems, complementing the teacher's work. A tutor
system builds a profile for each student and estimates their level of
knowledge, allowing the system to change the tutoring in real time, adjusting
it in order to interact more effectively with the student. Along the
years there were a few proposals in the field of geometry, such as
\emph{Advanced Geometry Tutor} (AGT)~\cite{matsuda2005},
\textit{AgentGeom}~\cite{cobo2007}, \textit{Baghera}
project~\cite{balacheff2003},
\textit{Cabri-G{\'e}om{\`e}tre}~\cite{luengo2005}, \textit{Geometry
  Explanation Tutor}~\cite{aleven2002},
\textit{geogebraTUTOR}~\cite{richard2007},
\emph{PCMAT}~\cite{martins2011} and Tutoriel Intelligent en G{\'e}om{\'e}trie
(\textit{TURING})~\cite{richard2007a}.

Close to our goal of a tailored axiom system approach is the
\emph{QED-tutrix} tutoring system. It was based on the \emph{geogebraTUTOR}
and \emph{TURING} systems. It has an automated deduction mechanism
running in the background and can assists the student in an
exploratory approach when solving geometry proofs.

\emph{QED-tutrix} is a system that guides the student, like a teacher,
through the learning process, weaving proofs in Euclidean
geometry. It is an interactive tool that guides students without
imposing restrictions on the order of their actions. It is an
intelligent tutoring system which assists students in proof solving by
providing hints while taking into account the student's cognitive
state.  The \emph{QED-tutrix}, using a automated deduction based on
the \emph{Prolog} rule-based logical query mechanism, builds the
\emph{Hypothesis, Properties, Definitions, Intermediate results and
  Conclusion graph} (HPDIC-graph). The HPDIC-graph contains all
possible proofs for a given problem, using a given set of axioms.  The
\emph{QED-tutrix} deduction engines uses a set of 707 properties and
definitions that were translated into inferences. No claim of
completeness is made, the only concern is soundness and the production
of proofs that are close to the practice of teachers and students in
secondary schools. The HPDIC-graph is constructed by forward chaining,
in a time-consuming process but, after being built it can be used to
provide the next step guidance that is expected to be given by
tutorial systems~\cite{Font2018,leduc2016}.

\section{Conclusions and Future Work}
\label{sec:conclusions}

The dynamic geometry systems already proved themselves in the classroom,
substituting the ``old'' ruler-and-compass construction. The usefulness
of deduction tools, for making conjectures and proving them, is still
to be established. Part of the problem may be in the curricula and the
lack of preparation of the teachers for the subject~\cite{Santos2021}.
The GATP themselves, are still more research tools, than tools that
can be used by a non-expert user, questions of efficiency and
readability of the proofs are to be addressed before a more wider use
can be considered.  Last but unfortunately, not in any means the
least, the complexity of the problem itself, proving, rigorously of
formally, it is a difficult task.

\paragraph*{Acknowledgement} The initial research had the help of
Carlos Reis from the Agrupamento de Escolas Lima de Freitas, Set{\'u}bal,
Portugal, \texttt{c.j.c.reis@gmail.com}.


\newcommand{\noopsort}[1]{}\newcommand{\singleletter}[1]{#1}

\appendix

\section{Lesson Plan for Axiomatic Plane Geometry}
\label{sec:classPlan}

\paragraph{Contents:} 

Euclid's Postulates:
\begin{description}
\item[Axiom I] A straight line may be drawn from any one point to any
  other point.
\item[Axiom II] Given two distinct points, there is a unique line that
  passes through them.
\item[Axiom III] A  terminated  line  can  be  produced  indefinitely.
\item[Axiom IV] A  circle  can  be  drawn  with  any  centre  and  any
  radius.
\item[Axiom V] That all right angles are equal to one another.
\item[Axiom VI] Through a given point $P$ not on a line $L$, there is
  one and only one line in the plane of $P$ and $L$ which does not
  meet $L$ (Playfair's version).
\end{description}

\medskip
Triangle congruence lemmas:
\begin{description}
\item[AAS] If two angles and a non-included side of one triangle are
  congruent to two angles and a non-included side of a second triangle,
  then the triangles are congruent.
\item[SAS] If two sides and the included angle of one triangle are
  congruent to two sides and the included angle of a second triangle,
  then the triangles are congruent.
\item[ASA] If two angles and the included side of one triangle are
  congruent to two angles and the included side of a second triangle,
  then the triangles are congruent.
\end{description}

\medskip
Two parallel or non-parallel lines are intersected by a transversal
lemmas.

\medskip
When two parallel lines are cut by a transversal, then:
\begin{description}
\item[Alternate Interior] the resulting alternate interior angles are
  congruent.
\item[Corresponding Angles] the resulting corresponding angles are congruent.
\item[Consecutive Interior] the pairs of consecutive interior angles
  formed are supplementary.
\end{description}

\paragraph{Goals:} To learn the Euclid's' Postulates. Learn how to
develop proofs, using Euclid's' postulates and lemmas about triangle
congruence and alternate interior, corresponding and consecutive
interior angles lemmas. Prove geometric conjectures using ICT tools.

\paragraph{Prerequisites.} Know the relative position of two lines in
the plane. 

\paragraph{Actions to Develop with the Students.} State Euclid's
axioms. Perform manual demonstrations. Perform formal (automatic)
proofs using appropriate software.

\paragraph{Material.} Paper and pen; calculator; worksheet; computer.

\paragraph{Summary.} Geometric Axioms and Lemmas. Manual and automatic
proofs.

\paragraph{Class Development.} The teacher explain Euclid's 5 axioms
and the lemmas. The teacher performs a worksheet with a manual
proof. The teacher does the proof manually and with the help of ICT
software.

\paragraph{Assessment of Learning.} The teacher requests, in writing,
that the students interpret the automatic demonstration, comparing it
with the manual demonstration. At the end of the class deliver the
conclusions drawn, for correction.

\section{Problem 1 --- Using Automated Deduction Tools}
\label{sec:ICTproblem1}

In this appendix a short presentation of the resolution of problem 1,
using the different automated deduction tools, described above, is
presented. If you are interested in all the files used in the
resolutions of problems 1 -- 4, please contact the authors.

\subsection{GeoGebra}
\label{sec:problem1geogebra}

In figure~\ref{fig:geogebraproblem1}, the blue dots, are
free points, we can use them to manipulate the construction, getting a
first, visual confirmation, of the truthfulness of the property. By
the use of the \texttt{Prove} command, we get a formal proof of the
conjecture, although without any proof script.

\begin{figure}[htbp!]
  \begin{center}
    \includegraphics[width=0.75\textwidth]{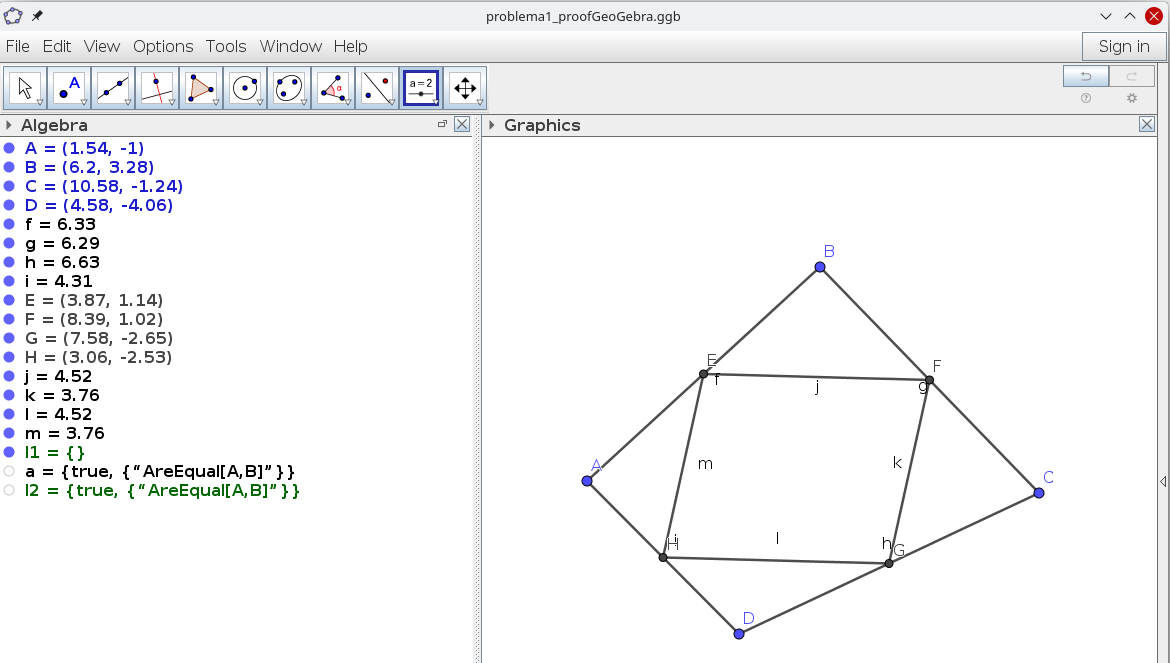}
  \end{center}
  \caption{{\GeoGebra} --- Problem 1}
  \label{fig:geogebraproblem1}
\end{figure}

\subsection{JGEX}
\label{sec:jgex}

As can be seen in figure~\ref{fig:jgexproblem1}, with {\JGEX} it is
possible to have the visual checks, by moving the free points around,
but also a formal proof, in this case using the geometry deductive
database (GDD) method. It is possible to see the connection
established between the construction and its rendering, and also
between the proof and its visual rendering.

\begin{figure}[htbp!]
  \begin{center}
    \includegraphics[width=0.45\textwidth]{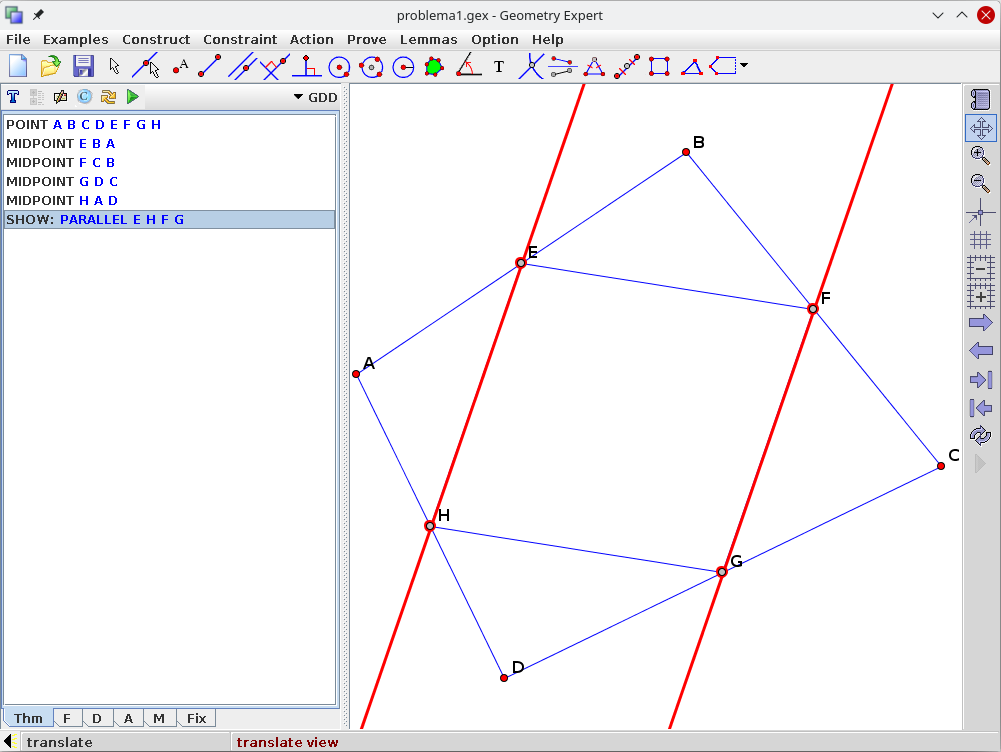}\quad
    \includegraphics[width=0.45\textwidth]{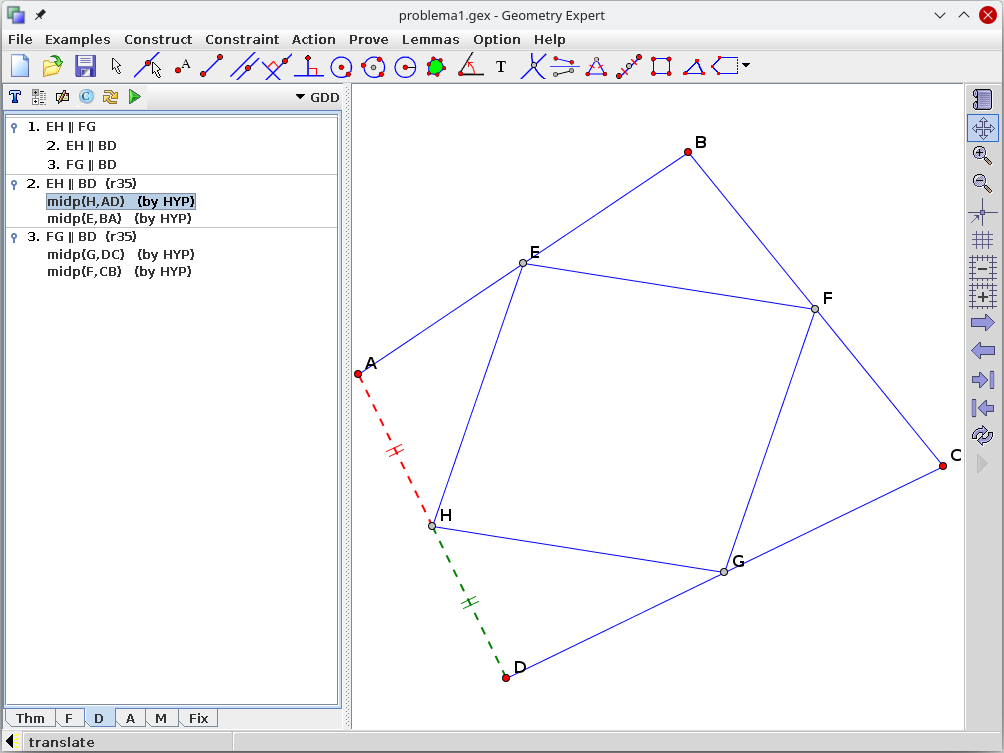}
  \end{center}
  \caption{{\JGEX} --- Problem 1 --- The conjecture (left), The
    Proof (right)}
  \label{fig:jgexproblem1}
\end{figure}

Unfortunately {\JGEX} it is not being developed at this moment, but given that
its authors took the decision of going open source, making available
all the code in a \emph{GitHub}
repository\footnote{\url{https://github.com/yezheng1981/Java-Geometry-Expert}}
it is possible that its development is resumed in the future.

\subsection{GCLC}
\label{sec:gclc}

{\GCLC} is more a GATP than a DGS. The geometric construction is
written in the \emph{GCL} language, it is rendered graphically, but
there are no free points that can be moved around, enabling a visual
verification of a given conjecture.

\begin{figure}[htbp!]
  \begin{center}
    \includegraphics[width=0.75\textwidth]{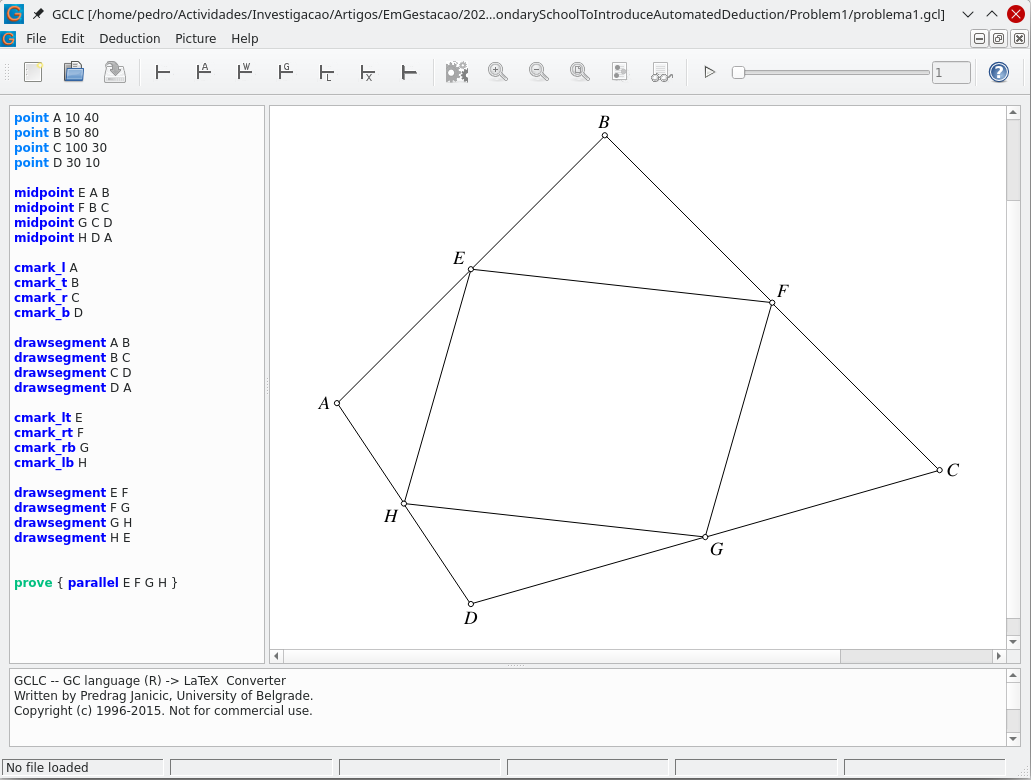}
  \end{center}
  \caption{GCLC --- Problem 1}
  \label{fig:gclcproblem1}
\end{figure}

Selecting one of its embedded GATP and processing the file, a
readable proof script (in \LaTeX) is produced. Unfortunately, 
readable here means, readable by experts, given that it is an
algebraic proof (Wu's method or Gr{\"o}bner basis method) or it is a
synthetic proof, but using a non-standard axiom system (area
method). In table~\ref{tbl:areamethodproblem1} an excerpt of the proof
produced by {\GCLC} area method prover is shown.

\setcounter{equation}{0}

\begin{table}[htbp!]
  \begin{displayproof}\footnotesize
    \proofstep{ S_{EGH}}{=}{ S_{FGH}}{by the statement}{-812.500000=-812.500000}
    \proofstep{ \left( S_{EGD} +  \left(\frac{1}{2}  \cdot  \left( S_{EGA} +  \left( -1  \cdot  S_{EGD}\right)\right)\right)\right)}{=}{ S_{FGH}}{by Lemma 29 (point $H$ eliminated)}{-812.500000=-812.500000}
    \proofstep{ \left( \left(\frac{1}{2}  \cdot  S_{EGD}\right) +  \left(\frac{1}{2}  \cdot  S_{EGA}\right)\right)}{=}{ S_{FGH}}{by algebraic simplifications}{-812.500000=-812.500000}
    \proofstep{ \left( \left(\frac{1}{2}  \cdot  S_{EGD}\right) +  \left(\frac{1}{2}  \cdot  S_{EGA}\right)\right)}{=}{ \left( S_{FGD} +  \left(\frac{1}{2}  \cdot  \left( S_{FGA} +  \left( -1  \cdot  S_{FGD}\right)\right)\right)\right)}{by Lemma 29 (point $H$ eliminated)}{-812.500000=-812.500000}
    \proofstep{ \left( \left(\frac{1}{2}  \cdot  S_{DEG}\right) +  \left(\frac{1}{2}  \cdot  S_{AEG}\right)\right)}{=}{ \left( S_{DFG} +  \left(\frac{1}{2}  \cdot  \left( S_{AFG} +  \left( -1  \cdot  S_{DFG}\right)\right)\right)\right)}{by geometric simplifications}{-812.500000=-812.500000}
    \proofstep{ \left( S_{DEG} +  S_{AEG}\right)}{=}{ \left( S_{DFG} +  S_{AFG}\right)}{by algebraic simplifications}{-1625.000000=-1625.000000}
    \proofstep{ \left( \left( S_{DEC} +  \left(\frac{1}{2}  \cdot  \left( S_{DED} +  \left( -1  \cdot  S_{DEC}\right)\right)\right)\right) +  S_{AEG}\right)}{=}{ \left( S_{DFG} +  S_{AFG}\right)}{by Lemma 29 (point $G$ eliminated)}{-1625.000000=-1625.000000}
    \proofstep{ \left( \left( S_{DEC} +  \left(\frac{1}{2}  \cdot  \left( 0  +  \left( -1  \cdot  S_{DEC}\right)\right)\right)\right) +  S_{AEG}\right)}{=}{ \left( S_{DFG} +  S_{AFG}\right)}{by geometric simplifications}{-1625.000000=-1625.000000}
    \proofstep{ \left( \left(\frac{1}{2}  \cdot  S_{DEC}\right) +  S_{AEG}\right)}{=}{ \left( S_{DFG} +  S_{AFG}\right)}{by algebraic simplifications}{-1625.000000=-1625.000000}
    \proofstep{ \left( \left(\frac{1}{2}  \cdot  S_{DEC}\right) +  \left( S_{AEC} +  \left(\frac{1}{2}  \cdot  \left( S_{AED} +  \left( -1  \cdot  S_{AEC}\right)\right)\right)\right)\right)}{=}{ \left( S_{DFG} +  S_{AFG}\right)}{by Lemma 29 (point $G$ eliminated)}{-1625.000000=-1625.000000}
  \end{displayproof}
  \caption{GCLC, Area Method Proof (excerpt) --- Problem 1}
  \label{tbl:areamethodproblem1} 
\end{table}

The proof is 4 pages long, with 35 steps. The prover gives the
following information at the end of the proof. ``Q.E.D.'' (or not proved,
if that was the case) . There are no ndg (non-degenerated)
conditions. Number of elimination proof steps: 12. Number of geometric
proof steps:       21. Number of algebraic proof steps:
72. Total number of proof steps:          105. Time spent by the
prover: 0.005 seconds.

\subsection{Prover9}
\label{sec:prover9}

To be able to use {\ProverNine} the full-angle method was converted to
first-order form, TPTP/FOF syntax~\cite{Sutcliffe2017} and the
geometric conjectures were also written in that language. A specific
filter was used to convert that to {\ProverNine} internal
syntax.

\begin{lstlisting}[language = FOF,frame=single,basicstyle=\footnotesize]
% ---Include Geometry Deductive Database Method axioms
include('geometryDeductiveDatabaseMethod.ax').

fof(tgtpproblema1,conjecture,( ! [ A,B,C,D] : 
   ( midp(E,A,B) &  midp(F,B,C) & midp(G,C,D) & midp(H,D,A) )
   =>
   ( para(E,F,G,H) ) ) ).
\end{lstlisting}

The prover took 0.02s to prove it, producing a proof script (total of
1076 lines), with a final proof, short and readable, but not by
secondary students nor by their teachers (without special training).

\begin{lstlisting}[language = ProverNine,frame=single,basicstyle=\footnotesize]
============================== PROOF =================================
% Proof 1 at 0.02 (+ 0.00) seconds.
% Length of proof is 7.
% Level of proof is 3.
% Maximum clause weight is 13.000.
% Given clauses 100.
64 (all A all B all C all D all M (midp(M,A,B) & midp(M,C,D) -> para(A,C,B,D))) 
   # label(ruleD63) # label(axiom) # label(non_clause).  [assumption].
95 (all A all B all C all D (midp(E,A,B) & midp(F,B,C) & midp(G,C,D) &
midp(H,D,A))) -> para(E,F,G,H) # label(tgtpproblema1) 
   # label(conjecture) # label(non_clause) # label(goal).  [goal].
166 -midp(A,B,C) | -midp(A,D,E) | para(B,D,C,E) 
   # label(ruleD63) # label(axiom).  [clausify(64)].
208 midp(c1,A,B) # label(tgtpproblema1) 
   # label(conjecture).  [deny(95)].
212 -para(c1,c2,c3,c4) 
   # label(tgtpproblema1) # label(conjecture).  [deny(95)].
326 para(A,B,C,D).  [resolve(208,a,166,b),unit_del(a,208)].
327 $F.  [resolve(326,a,212,a)].
============================== end of proof ==========================
\end{lstlisting}

The use of FOF syntax allows the use of other generic ATP, for example,
\emph{Vampire}\footnote{\url{https://vprover.github.io/}} was also
used, with similar results.

\end{document}